\documentclass{article}

\PassOptionsToPackage{numbers}{natbib}

\usepackage[final]{neurips_glfrontiers_2022}

\usepackage[utf8]{inputenc} 
\usepackage[T1]{fontenc}    
\usepackage[colorlinks=true,allcolors=blue]{hyperref}
\usepackage{url}            
\usepackage{booktabs}       
\usepackage{amsfonts}       
\usepackage{nicefrac}       
\usepackage{microtype}      
\usepackage{xcolor}         

\usepackage{graphicx}
\usepackage{amsmath}

\usepackage{bbm}
\usepackage{makecell}
\usepackage{multirow}
\usepackage{arydshln}
\usepackage{enumitem}
\usepackage{caption}
\usepackage{amsthm}

\usepackage{float}
\usepackage{wrapfig}

\newtheorem{theorem}{Theorem}
\newtheorem{lemma}{Lemma}

\newtheorem*{appendix_theorem}{Theorem 1}

\usepackage[scientific-notation=true]{siunitx}

\usepackage{algorithm}
\usepackage{algorithmic}

\title{SPGP: Structure Prototype Guided Graph Pooling}

\author{%
Sangseon Lee$^1$ \quad Dohoon Lee$^2$ \quad Yinhua Piao$^3$ \quad Sun Kim$^{3,4,5,6}$\thanks{Corresponding author} \\
$^1$Institute of Computer Technology, Seoul National University, Korea \\
$^2$BK21 FOUR Intelligence Computing, Seoul National University, Korea \\
$^3$Department of Computer Science and Engineering, Seoul National University, Korea \\
$^4$Interdisciplinary Program in Bioinformatics, Seoul National University, Korea \\
$^5$Interdisciplinary Program in Artificial Intelligence, Seoul National University, Korea \\
$^6$AIGENDRUG Co. Ltd., Korea\\
\texttt{\{sangseon486, apap7, 2018-27910, sunkim.bioinfo\}@snu.ac.kr}
}

\begin{document}

\maketitle

\begin{abstract}
While graph neural networks (GNNs) have been successful for node classification tasks and link prediction tasks in graph, learning graph-level representations still remains a challenge. 
For the graph-level representation, it is important to learn both representation of neighboring nodes, i.e., aggregation, and graph structural information.
A number of graph pooling methods have been developed for this goal. However, most of the existing pooling methods utilize k-hop neighborhood without considering explicit structural information in a graph.
In this paper, we propose Structure Prototype Guided Pooling (SPGP) that utilizes prior graph structures to overcome the limitation.
SPGP formulates graph structures as learnable prototype vectors and computes the affinity between nodes and prototype vectors.
This leads to a novel node scoring scheme that prioritizes informative nodes while encapsulating the useful structures of the graph.
Our experimental results show that SPGP outperforms state-of-the-art graph pooling methods on graph classification benchmark datasets in both accuracy and scalability.
\end{abstract}

\section{Introduction}

Graph neural networks (GNNs) have been successfully applied to graph-structured data for node classification tasks~\cite{kipf2016semi,hamilton2017inductive,xu2018powerful} and link prediction tasks~\cite{zhang2018link,yun2021neo}. 
Most of the existing GNNs use graph convolution to aggregate neighbor information at the node-level.
However, graph-level representation requires additional condensation of an entire graph. 
For this,  a number of graph pooling methods have been developed and used for learning a compact embedding of an entire graph.
Capturing global or hierarchical structures in a graph is important to learn accurate representations of  graphs of diverse characteristics such as size, connectivity, or density.
Due to the structural diversity of graphs, developing efficient and effective graph pooling methods remains a challenge. 
The main difficulty is how to consider node features and graph structures simultaneously.

To address these issues of the graph pooling methods, a number of methods have been developed to capture global or hierarchical structures of a graph.
Global pooling methods~\cite{vinyals2015order,li2015gated,zhang2018end} obtain graph-level representations directly from node-level representations via permutation invariant readout functions.
They handle global interactions between nodes, but they are inherently flat and are not powerful enough to learn the hierarchical structures of the graph.
Recently, hierarchical graph pooling methods have been proposed, reducing gradually a graph to smaller pooled graphs while maintaining essential structures of a graph.
To construct the pooled graphs, a number of techniques have been explored based on 
a concept of the clustering algorithms~\cite{ying2018hierarchical,Ma:2019:GCN:3292500.3330982,yuan2020structpool}, a priority of node features~\cite{gao2019graph} and integration of structural information~\cite{lee2019self,zhang2020structure,gao2021ipool}, and end-to-end structure learning~\cite{ranjan2020asap,zhang2019hierarchical}.

However, these approaches  extract graph structural information from k-hop neighborhood, which faces two technical issues.
First, graph pooling based on  k-hop neighborhood depends on $k$, which is often an arbitrary value. 
When the value of $k$ is small, the receptive field of a  k-hop neighborhood is relatively limited to local information which does not encapsulate the global structural information of a graph~\cite{garg2020generalization}.
On the other hand, with a large value of $k$,  irrelevant nodes can be easily included within the receptive field. 
This observation can be even worse in scale-free networks such as technological and biological networks~\cite{ broido2019scale,xu2018representation,oono2019graph}.
Second, explicit graph structures, that can be  difficult to represent with k-hop neighborhood, are overlooked.
For example, biconnected component (BCC),  one of the important graph structures in graph theory, is an important candidate graph structure.
Mutual interactions within the BCC can cover long-range information with low noise.
Thus, BCC is widely used to effectively capture the global structural information of a graph in various tasks such as weighted vertex cover or minimum cost flow~\cite{hong2018bc,shao2020efficient, hochbaum1993should}.
Another promising candidate is a set of cliques.
It is well known that the clique set is closely related to the global topology of the graph~\cite{otter2017roadmap,petri2013topological}.
Due to these global topology-related properties, cliques are being utilized for predicting graph-level properties in various fields (e.g., in chemistry~\cite{national1995mathematical} and biology~\cite{srihari2017computational}).
These observations motivate us to develop a graph pooling method that utilizes explicit graph structures beyond k-hop neighborhood.

\begin{figure*}[t]
\centering
\includegraphics[width=\textwidth]{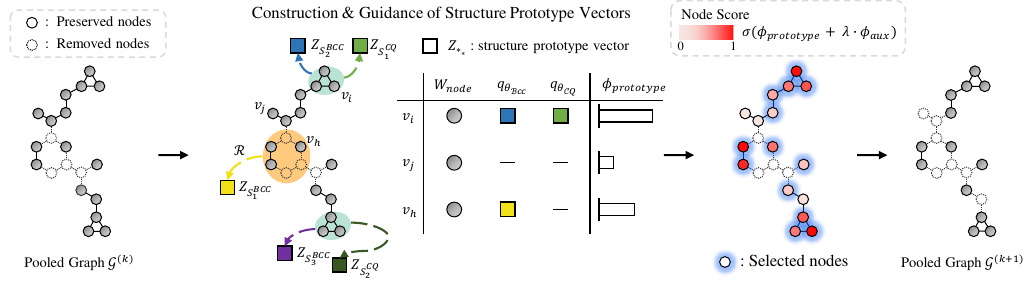} 
\caption{Concept of SPGP:  SPGP first constructs the prototype vectors $Z_{\mathcal{S}^{BCC}_*}$ and $Z_{\mathcal{S}^{CQ}_*}$ by utilizing biconnected components and cliques as the structure prototypes. The global structural importance of a node is calculated based on the information of the node itself ($W_{node}$) and the affinities between the node and the prototype vectors, which are measured by relation modules $q_{\theta_s}$. Finally, with an auxiliary score, SPGP extracts important nodes to generate the pooled graph.
}
\label{fig:overview}
\end{figure*}

In this paper, we propose the Structure Prototype Guided Pooling (SPGP) that extracts informative nodes and obtains more accurate graph representations by utilizing the overlooked graph structures related to the global information of the graph. 
Given the prior graph structures, SPGP considers them as structure prototypes and computes their information through learnable prototype vectors in an end-to-end fashion.
To guide the pooling, SPGP prioritizes nodes by measuring the affinity between the nodes and the prototype vectors.
We extensively validate SPGP on real-world graphs in widely used benchmark datasets including the large-scale Open Graph Benchmark (OGB) dataset, where SPGP outperforms existing graph pooling methods in both accuracy and scalability.
Our paper’s main contributions are:
\begin{itemize}
	\item We propose a new graph pooling that utilizes prior graph structures, which are hard to represent as k-hop neighborhood, to make more accurate graph representations with global structural information.
    \item We show that SPGP can capture structural information better than the existing pooling methods. Additionally, SPGP efficiently addresses the isolated node problem of graph pooling, without computationally expensive procedures.
    \item We evaluate SPGP on widely used benchmark datasets for graph classification under the extensive experiments such as 20 different seeds with 10-fold cross-validation on each seed. We also demonstrate SPGP outperforms over the state-of-the-art graph pooling methods.
\end{itemize}

\section{Related Works}
\subsection{Graph Neural Networks}
Recently, various GNNs have been proposed to compensate for the shortcomings of graph representation learning through matrix factorization~\cite{ou2016asymmetric} or random walk~\cite{grover2016node2vec}.
These GNNs can be categorized into two domains: spectral and spatial.
Spectral methods generally utilize graph Fourier transform with approximations for handling large graphs~\cite{defferrard2016convolutional,kipf2016semi}.
The early proposed GCN~\cite{kipf2016semi} bridges the gap between "spatial" and "spectral"-based methods, due to its convolutional form can be seen as a first-order approximation of ChebNet~\cite{defferrard2016convolutional}. GIN~\cite{xu2018powerful} provides a theoretical proof of the best way to propagate and aggregate information and adds a linear mapping of propagated information and its own information to GCN, which enhances the learning power of graph neural networks. GAT~\cite{velivckovic2018graph} proposes a new way to learn the contributions of k-hop local neighborhood nodes to the target node and obtains good performance in graph benchmark tasks. There are also many approaches that try diverse aggregation approaches to explore more effective strategies, such as graphSage~\cite{hamilton2017inductive}, which aggregates neighborhood information by randomly sampling local neighborhood nodes and uses different methods (mean, lstm, gcn, pool) to obtain a more effective map neural network architecture.

\subsection{Graph Pooling}
The pooling operation performs downsampling on the graph to learn the graph representation efficiently.
There are two branches of the pooling operations: global pooling and hierarchical pooling.
The global pooling methods~\cite{vinyals2015order,li2015gated,zhang2018end} summarize the entire graph at once by permutation invariant readout functions.
Since these methods do not capture the hierarchical structures in the graph, hierarchical pooling methods have been proposed.
DiffPool~\cite{ying2018hierarchical} uses node embeddings to smoothly assign nodes into clusters to construct a new graph topology.
TopK Pooling~\cite{gao2019graph} selects the top k-th nodes by scoring the nodes with their hidden representations.
SAGPool~\cite{lee2019self} is an extension of TopK Pooling that calculates node scores by considering neighborhood information via an additional GNN.
ASAP~\cite{ranjan2020asap} clusters local subgraphs by capturing local extremum information with maintaining graph connectivity.
GSAPool~\cite{zhang2020structure} learns structural feature-based information by mixing node information and structural information.
HGP-SL~\cite{zhang2019hierarchical} performs attention based structure learning on the pooled graph.
Unlike these methods, which focus only on k-hop neighborhood, our approach uses more global structural information through prior graph structures.

\section{SPGP: Structure Prototype Guided Pooling}
In this section, we introduce our proposed method SPGP for graph pooling as illustrated in Figure \ref{fig:overview}.
Given prior graph structures such as BCC or cliques, SPGP computes the significance of the graph structures in the form of learnable prototype vectors using node representations with contextual information.
The importance of these structure prototype vectors is then distributed to the nodes belonging to the graph structures.
Further, with an auxiliary node scoring scheme, SPGP prioritizes the structural importance of nodes to select the top scoring nodes for creating the pooled graph.
In the following, we describe the SPGP approach in detail.

\subsection{Notations and Problem Statement}
The problem of graph classification considers a set of labeled graphs $D = \{(\mathcal{G}_1, y_1), \dots , (\mathcal{G}_M, y_M)\}$, and we aim to learn a mapping $f: G \rightarrow Y$, where $G$ is the set of graphs and $Y$ is the set of graph labels.
For an arbitrary graph $\mathcal{G} = (\mathcal{V}, \mathcal{E}, X)$, the node set $\mathcal{V}$ and the edge set $\mathcal{E}$ have $n$ nodes and $e$ edges, respectively.
$X \in \mathbb{R}^{n \times \delta}$ represents the node feature matrix, where $\delta$ is the dimension of node features.
$A \in \mathbb{R}^{n \times n}$ denotes the adjacency matrix.
After the pooling operation on layer $k$, the number of nodes is changed as $n^{(k)}$, and the pooled graph $\mathcal{G}^{(k)}$ has the hidden representation matrix $H^{(k)} \in \mathbb{R}^{n^{(k)} \times d}$ and the adjacency matrix $A^{(k)} \in \mathbb{R}^{n^{(k)} \times n^{(k)}}$, where $d$ is the dimension of hidden representations and $\mathcal{G}^{(0)} = \mathcal{G}$ and $H^{(0)} = X$.
For SPGP, one or multiple types of prior graph structures can be used.
Given a graph $\mathcal{G}$, the $s$-th type of graph structure is denoted as the structure prototype $SP_{s} = \{\mathcal{S}^s_{t} \mid \mathcal{S}^s_{t} \subset \mathcal{V}, t= 1 \dots T\}$, where $\mathcal{G}$ has $T$ graph structures of type $s$ and $\mathcal{S}^s_{t}$ is a set of nodes belonging to the $t$-th graph structure of $SP_{s}$.

\subsection{Guided Pooling with Structure Prototype Vectors}
The core idea of graph pooling is to select the nodes that are important in terms of the graph as a whole.
Unlike existing graph pooling methods, SPGP guides the learning process through the structure prototypes $SP_s$.
SPGP learns a \textit{node selector} using structure prototype vectors which constitute the graph structures based on the node representations with contextual information.

\subsubsection{Node Selection through Structure Prototype Vectors}

Given a graph $\mathcal{G}^{(k)}$, SPGP learns the node importance by utilizing the structure prototypes $SP_{s}$ and the node representation $\widetilde{h}^{(k)}_v$, which contains the node contextual information (See Appendix \ref{appendix:context_info}).
The global structural information contained in the $SP_{s}$ is distributed to each node through the learnable \textit{structure prototype vectors} $Z_{\mathcal{S}^s_t}$, which can be expressed as:
(For simplicity, we omit the superscript $k$ as below.)
\begin{equation}
    Z_{\mathcal{S}^s_t} = \mathcal{R}\left(\{\widetilde{h}_v \mid v \in \mathcal{S}^s_t\}\right) 
    \label{eq:prototype_R}
\end{equation}
where $\mathcal{R}$ is an aggregation function that integrates the representation of nodes belonging to $\mathcal{S}^s_t$ of the structure prototype $SP_s$.
$\mathcal{R}$ can be any permutation-invariant methods such as simple average or attention mechanism.
For computational efficiency, we simply utilize element-wise max operation to obtain the prototype vectors, i.e., $Z_{\mathcal{S}^s_t, j} = max_{v \in \mathcal{S}^s_{t}}\widetilde{h}_{v,j}$.
In practice, despite the simple aggregator, the information in the global structure is well represented.

\paragraph{\textit{Prototype Score}}
Then, using the structure prototype vectors $Z_{\mathcal{S}^s_t}$, a prototype  score is calculated as:
\begin{equation}
    \phi_{prototype}(v) = \sum_s q_{\theta_{s}}\left(\{Z_{\mathcal{S}^s_t} \mid v \in \mathcal{S}^s_t\}, \widetilde{h}_v \right) + \boldsymbol{W}_{node} \, \widetilde{h}_v
    \label{eq:prototype_score}
\end{equation}
where $\boldsymbol{W}_{node} \in \mathbb{R}^{d \times 1}$ is a learnable parameter for extracting information of the node itself.
$q_{\theta_{s}}(.)$ is relation module of the structure prototype $SP_s$ that learns affinities between the prototype vectors and the node.
Specifically, the relation module $q_{\theta_{s}}$ is calculated as:
\begin{equation}
    q_{\theta_{s}}\left(\{Z_{\mathcal{S}^s_t} \mid v \in \mathcal{S}^s_t\}, \widetilde{h}_v \right)  =  \boldsymbol{W}_{s} \, \textnormal{Concat}\left(\sum_{z \in Z_v} z, \;\;\widetilde{h}_v\right) 
    \label{eq:prototype}
\end{equation}
where $Z_v$ is denoted as $\{Z_{\mathcal{S}^s_t} \mid v \in \mathcal{S}^s_t\}$. $\boldsymbol{W}_{s} \in \mathbb{R}^{2d \times 1}$ is a learnable parameter for extracting relationship between the prototypes and the node. 
If the node $v$ does not belong to the structure prototype $SP_s$,  the relation module $q_{\theta_{s}}$ generates zero, which means no affinity between the node and the structure prototype.

\paragraph{\textit{Auxiliary Score}}
Inspired by ~\cite{gao2021ipool}, we utilize neighborhood information gain as an auxiliary score to obtain local structure information.
Different from the previous work, we aggregate representations of the neighborhood with a learnable projection function.
\begin{equation}
    \phi_{aux}(v) = {\lVert}\widetilde{h}_v - \sum_{u \in N_v}\boldsymbol{W}_{aux} \widetilde{h}_u{\rVert}_1
    \label{eq:aux_score}
\end{equation}
$\boldsymbol{W}_{aux} \in \mathbb{R}^{d \times d}$ is a learnable projection that adjusts the discrepancy between the node $v$ and the neighborhood.
${\lVert}\cdot{\rVert}_1$ is $L_1$-norm or Manhattan norm.

\paragraph{\textit{Total Score}}
Finally, the node score $\phi(v)$ of SPGP is expressed as the prototype score with the auxiliary score to learn the structural information of the node $v$:
\begin{equation}
    \phi(v) = \sigma\left(\phi_{prototype}(v) + \lambda \, \phi_{aux}(v)\right)
\end{equation}
where $\sigma$ denotes the activation function such as sigmoid or ReLU.
$\lambda \in [0, 1]$ is the regularization for the auxiliary score.

\subsubsection{Generation of Pooled Graph} 
Based on the node scores $\phi(.)$, SPGP generates the pooled graph $\mathcal{G}^{(k+1)}$ at layer $k+1$  as:
\begin{align}
    \textrm{idx} & = \textrm{Top-K}(\phi(v^{(k)}), p) \label{eq:idx}\\
    H^{k+1}  = H^{k}\:[\textrm{idx}, &\;:]  \odot \phi(.), \;\;\;\;
    A^{k+1}  = A^{k}\:[\textrm{idx}, \;\textrm{idx}] \label{eq:filter}
\end{align}
where $p \in (0, 1]$ is the pooling ratio.
Top-K~\cite{gao2019graph} selects the top $\lceil p n^{(k)}_i \rceil$ nodes based on the score $\phi$. 
The $\textrm{idx}$ is indices of selected nodes for the pooled graph.
$\odot$ is broadcasted hadamard product.
Note that, the end product of node scores $\phi(.)$ is a discrete selection of nodes, which is not trainable by back-propagation.
Therefore, the score $\phi(.)$ is multiplied on hidden representations to make the SPGP learnable in an end-to-end fashion.

\subsection{Well-studied Structure Prototypes: BCC and Cliques}
SPGP can use any graph structure related to global information of the graph.
In this paper, as an example of the structure prototypes, we utilize two well-studied structures related to the overall graph topology in graph theory: \textit{Biconnected Component (BCC)} and \textit{Cliques}.

Given a graph $\mathcal{G}$, a set of biconnected components $SP_{BCC}$ and a set of cliques $SP_{CQ}$ are obtained by the graph traversal algorithms~\cite{hopcroft1973algorithm,zhang2005genome}.
Specifically, for biconnected components, we discard the biconnected components with fewer than three nodes to avoid simple biconnected components, i.e., single edges.
Then, we define the structure prototype for biconnected components of $\mathcal{G}$ as:
\begin{equation}
SP_{BCC} = \{\mathcal{S}^{BCC}_{p} \mid \mathcal{S}^{BCC}_{p} \subset \mathcal{V}, |\mathcal{S}^{BCC}_{p}| \geq 3, p = 1 \dots P \}
\end{equation}
where $\mathcal{G}$ has $P$ biconnected components and $\mathcal{S}^{BCC}_{p}$ is a set of nodes belonging to the $p$-th biconnected component of $\mathcal{G}$.

For cliques, we also discard cliques with less than three nodes.
A set of cliques can characterize the overall topology of the graph~\cite{petri2013topological}, so if adjacent cliques share more than half of the nodes, they are merged.
Then, we define the structure prototype for cliques of $\mathcal{G}$ as:
\begin{equation}
SP_{CQ} = \{\mathcal{S}^{CQ}_{q} \mid \mathcal{S}^{CQ}_{q} \subset \mathcal{V}, |\mathcal{S}^{CQ}_{q}| \geq 3, q=1 \dots Q\}
\end{equation}
where $\mathcal{G}$ has $Q$ the merged cliques and $\mathcal{S}^{CQ}_{q}$ is a set of nodes belonging to the $q$-th cliques of $\mathcal{G}$.

The time complexity of finding BCC is $\mathcal{O}(|\mathcal{V}|+|\mathcal{E}|)$.
Cliques can be efficiently searched on each biconnected component, not the entire graph~\cite{hochbaum1993should}.
In addition, graph traversal is performed only once during preprocessing, not for model training.

\section{Theoretical and Computational Analysis} \label{sec:theorectical}

\subsection{SPGP can discriminate structural role of nodes}
We investigate  the theoretical understanding on the importance of learning structural information.
Following existing studies on the importance of graph structure~\cite{you2019position,nikolentzos2020k}, we assume that all nodes have the same feature vector, which means that the only information we can learn from the graph is structural information.
In general, nodes belonging to a local level subgraph tend to have same or similar feature vectors due to homophily~\cite{mcpherson2001birds,evtushenko2021paradox}.
In this case (e.g., naphthalene), it is difficult to learn the structural role between nodes in the subgraph using node features.
Note that if the structural meaning of the nodes on the graph is different, efficient graph pooling should be able to identify the difference in the amount of information of each node.

\begin{theorem}
Given a regular graph that has prior graph structures,
\begin{enumerate}[label=(\roman*)]
    \item Standard k-hop neighborhood based graph pooling methods cannot distinguish the nodes.
    \item SPGP can assign different node scores based on the structural role of nodes.
\end{enumerate}
\label{theorem1}
\end{theorem}

\begin{proof}
See Appendix \ref{appendix:proof} for proof.
\end{proof}

Motivated by the above analysis, we believe that providing the pooling with prior graph structures beyond k-hop neighborhood can better learn the structural information of the graph.
Therefore, to extract and model the structure prototype information in an end-to-end fashion, we formulate SPGP based on the learnable structure prototype vectors.

\begin{wrapfigure}{r}{0.4\textwidth}
    \vspace{-20pt}
    \centering
    \includegraphics[width=0.38\textwidth]{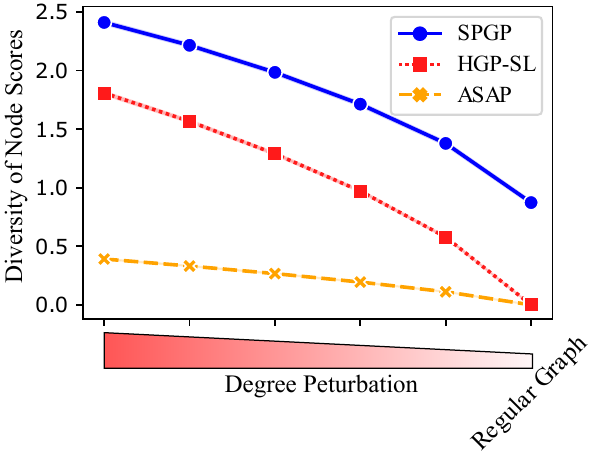}
    \vspace{-5pt}
    \caption{Node Score's diversity of SPGP, HGP-SL, and ASAP on regular graphs with degree perturbations.}
    \vspace{-30pt}
  \label{fig:regular_graph}
\end{wrapfigure}

To further support the observation in more general situations, we measure the diversity of node scores in the regular graphs with degree perturbations.
That is, we generate more realistic graphs with structural diversity from regular graphs (See Appendix \ref{appendix:regular_test} for details).
Figure \ref{fig:regular_graph} shows that, as in Theorem \ref{theorem1}, SPGP learns the structural roles of nodes in regular graphs, but the other pooling methods have failed.
Furthermore, as the degree perturbation increases, that is, as the distribution of node degree on the graph diversifies, SPGP captures the structural role of nodes better than the other methods.

\subsection{Efficient Handling of Isolated Nodes} \label{sec:complexity}

SPGP is a node selection based graph pooling which can cause isolated nodes during the pooling.
Since  isolated nodes cannot aggregate the information of neighbors, they may result in loss of graph structural information and degenerate graph classification performances.
To address this issue, the existing graph pooling methods, such ASAP~\cite{ranjan2020asap} and HGP-SL~\cite{zhang2019hierarchical},  perform `structure learning' that refines an adjacency matrix of a pooled graph.
However, this operation is an additional task performed after graph pooling and it is computationally expensive.
The computational complexity of ASAP is  $O(n^3)$ time at a cost of $O(n^2)$ space.
HGP-SL is $O(dn^2)$ time at a cost of $O(n^2)$ space.
Meanwhile, our SPGP does not require an additional step for addressing the isolated nodes.
During the graph pooling, Equation \ref{eq:prototype} aggregates and updates structural information of nodes belonging to structure prototypes, even if they are isolated.
Furthermore, the computational complexity of Equation \ref{eq:prototype} is $O(sdn)$ time at a cost of $O(dn)$ space, $n \gg s, d$, which is more efficient for the large-scale graphs.
Total computational complexity of SPGP is discussed in Appendix \ref{appendix:complexity}.

\section{Experiment Setup}
\subsection{Datasets} \label{sec:dataset}

For fair comparisons, we use the widely used benchmark datasets for graph classification.
Specifically, datasets from biochemical domains of TUDataset~\cite{Morris+2020} (DD, PROTEINS, NCI1, NCI109, and FRANKENSTEIN; ranging from 1K to 4K graphs) are utilized for performance evaluation.
Furthermore, two large-scale benchmark datasets (ogbg-molhiv with 41.1K graphs and ogbg-molpcba with 437.9K graphs) are utilized from Open Graph Benchmark (OGB)~\cite{hu2020ogb}.
Meanwhile, in real-world datasets, especially in the biochemical domains, most graphs contain more than one BCC or clique.
For example, in all datasets except FRANKENSTEIN, more than 97\% of graphs have these structures
(DD: 100\%, PROTEINS: 99.9\%, NCI1: 98.0\%, NCI109: 97.6\%, ogbg-molhiv: 100\%, ogbg-molpcba: 99.9\%).
Even in the chimera dataset, FRANKENSTEIN, the ratio is 82.6\%.
Detailed information about the datasets is described in the Appendix \ref{appendix:dataset_detail}.

\subsection{Baselines}
We compare the performance of SPGP with state-of-the-art global pooling (i.e., Set2Set~\cite{vinyals2015order}, Global-Attention~\cite{li2015gated}, and SortPool~\cite{zhang2018end})) or hierarchical graph pooling methods (i.e., DiffPool~\cite{ying2018hierarchical}, TopK~\cite{gao2019graph}, SAGPool~\cite{lee2019self}, ASAP~\cite{ranjan2020asap}, GSAPool~\cite{zhang2020structure}, HGP-SL~\cite{zhang2019hierarchical}, and iPool~\cite{gao2021ipool}).
We report the results, except for GSAPool, HGP-SL, $\textrm{iPool}_{greedy}$, and $\textrm{iPool}_{local}$, from~\cite{ranjan2020asap} and measure the performance of GSAPool and HGP-SL based on the released source codes. We implement both of iPool variants by modifying the code of HGP-SL.
Brief explanations of the hierarchical graph pooling methods are described in Appendix \ref{appendix:baselines}

\subsection{Experimental Details}
In case of TUDataset, following~\cite{ranjan2020asap}, we adopt the extensive experimental designs for performance evaluation.
In short, the performances are averaged over the results based on 20 different random seeds.
In addition, on each seed, 10-fold cross-validation is performed and data split configurations are the same in~\cite{ranjan2020asap}.
In case of OGB datasets, we utilize the scaffold splitting scheme and report the average results on the 20 different random seeds.
Hyperparameters of SPGP are selected by grid search on the validation data.
Settings of hyperparameter values are shown in the Appendix \ref{appendix:training}.

\section{Results}
This section attempts to answer the following questions:
\begin{enumerate}[label=\textbf{Q{{\arabic*}}}]
    \item Whether can prior graph structures improve graph classification performances of SPGP compared to baselines? (Section \ref{sec:q1})
    \item How does a regularization $\lambda$ affect the performances over various datasets? (Section \ref{sec:q2})
    \item Can SPGP capture meaningful structures during pooling? (Section \ref{sec:q3})
    \item How does SPGP perform compared to baselines on large-scale datasets? (Section \ref{sec:q4})
\end{enumerate}

\begin{table}[t]
\caption{Test Accuracy (\%) performance on TUDataset. Average accuracy and standard deviation is reported for 20 random seeds. The best result and its comparable results ($p > 0.05$) from t-test for each dataset are shown in bold. The second best results on each category are shown as underlined. Compared with the previous graph pooling methods, the proposed SPGP outperforms the baselines.}
\label{tab:performance_comparison}
\centering
\resizebox{\textwidth}{!}{
\begin{tabular}{ccccccc}
\toprule
Category & Methods & DD & PROTEINS  & NCI1 & NCI109 & FRANKENSTEIN \\ \midrule
Baseline & GCN & $68.33 \pm 1.30$ & $73.83 \pm 0.33$ & $73.92 \pm 0.43$ & $72.77 \pm 0.57$ & $60.47 \pm 0.62$ \\ \midrule
\multirow{3}{*}{Global} & Set2Set & $71.60 \pm 0.87$ & $72.16 \pm 0.43$ & $66.97 \pm 0.74$ & $61.04 \pm 2.69$ & $61.46 \pm 0.47$ \\
& Global-Attention & $71.38 \pm 0.78$ & $71.87 \pm 0.60$ & $\underline{69.00 \pm 0.49}$ & $67.87 \pm 0.40$ & $61.31 \pm 0.41$ \\
& SortPool & $\underline{71.87 \pm 0.96}$ & $\underline{73.91 \pm 0.72}$ & $68.74 \pm 1.07$ & $\underline{68.59 \pm 0.67}$ & $\underline{63.44 \pm 0.65}$ \\ \midrule
\multirow{11}{*}{Hierarchical}& DiffPool & $66.95 \pm 2.41$ & $68.20 \pm 2.02$ & $62.32 \pm 1.90$ & $61.98 \pm 1.98$ & $60.60 \pm 1.62$ \\
& TopK & $75.01 \pm 0.86$ & $71.10 \pm 0.90$ & $67.02 \pm 2.25$ & $66.12 \pm 1.60$ & $61.46 \pm 0.84$ \\
& SAGPool & $76.45 \pm 0.97$ & $71.86 \pm 0.97$ & $67.45 \pm 1.11$ & $67.86 \pm 1.41$ & $61.73 \pm 0.76$ \\
& ASAP & $\underline{76.87 \pm 0.70}$ & $\underline{74.19 \pm 0.79}$ & $71.48 \pm 0.42$ & $70.07 \pm 0.55$ & $\underline{66.26 \pm 0.47}$ \\
& GSAPool & $76.07 \pm 0.81$ & $73.53 \pm 0.74$ & $70.98 \pm 1.20$ & $70.68 \pm 1.04$ & $60.21 \pm 0.69$ \\  
& HGP-SL & $75.16 \pm 0.69$ & $74.09 \pm 0.84$ & $75.97 \pm 0.40$ & $74.27 \pm 0.60$ & $63.80 \pm 0.50$ \\ 
& $\textrm{iPool}_{greedy}$ & $75.39 \pm 0.76$ & $73.12 \pm 0.55$ & $76.03 \pm 0.60$ & $\underline{74.96 \pm 0.58}$ & $62.16 \pm 0.35$ \\ 
& $\textrm{iPool}_{local}$ & $75.22 \pm 0.70$ & $73.56 \pm 0.59$ & $\underline{76.49 \pm 0.62}$ & $74.64 \pm 0.39$ & $62.24 \pm 0.55$ \\  \cmidrule{2-7} 
& $\textnormal{SPGP}_{BCC}$ &  $76.01 \pm 0.92$ & $74.10 \pm 0.66$ & $77.36 \pm 0.31$ & $76.89 \pm 0.35$ & $66.86 \pm 0.67$ \\ 
& $\textnormal{SPGP}_{CQ}$ &  $76.77 \pm 0.80$ & $\boldsymbol{75.04 \pm 0.76}$ & $78.57 \pm 0.34$ & $\boldsymbol{77.23 \pm 0.48}$ & $67.72 \pm 0.51$ \\ 
& $\textnormal{SPGP}_{BCC+CQ}$  & $\boldsymbol{77.76 \pm 0.73}$ & $\boldsymbol{75.20 \pm 0.74}$ & $\boldsymbol{78.90 \pm 0.27}$ & $\boldsymbol{77.27 \pm 0.32}$ & $\boldsymbol{68.92 \pm 0.71}$ \\ \bottomrule
\end{tabular}
}
\end{table}

\subsection{Performance Comparison on TUDataset (Q1)} \label{sec:q1}


Table \ref{tab:performance_comparison} reports the experimental results of SPGP with baseline models on TUDataset.
From the table, we find that our $\textnormal{SPGP}_{BCC+CQ}$ significantly outperforms all baselines in terms of graph classification accuracy.
$\textnormal{SPGP}_{BCC+CQ}$ is up to 9.90\% and 2.66\% better than the second-best method in global pooling and hierarchical pooling baselines, respectively. 
We notice that the guided pooling using prior graph structures based on structure prototypes, in SPGP, is the main contribution to performance improvements against baselines.

To further investigate the effects of structure prototypes, we conduct ablation studies on SPGP with only one structure prototype.
(Due to space limit, an ablation study for without structure prototypes is discussed in Appendix \ref{appendix:ablation_proto}.)
In Table \ref{tab:performance_comparison}, SPGP$_{BCC}$ and SPGP$_{CQ}$ are variants of SPGP that use only BCCs and cliques respectively.
Interestingly, both variants of SPGP achieve better classification performances than all baselines except ASAP on DD dataset, which is also comparable.
This means that providing additional graph structural information to the model during the graph pooling is useful for learning more accurate graph-level representations.
Also, since the set of cliques can be either the BCC itself or a part of the BCC, the two graph structures can be used together as structure prototypes to provide nodes with more valuable structural information.
For this reason, $\textnormal{SPGP}_{BCC+CQ}$ achieves better performances than SPGP$_{BCC}$ and SPGP$_{CQ}$.

\subsection{Effects of Regularization $\lambda$ (Q2)} \label{sec:q2}
We investigate the effects of regularization $\lambda$ on the performance of TUDataset.
We use the $\textnormal{SPGP}_{BCC+CQ}$ model and the $\lambda$ in the range 0.0, 0.2, 0.4, 0.6, 0.8, 1.0.
Figure \ref{fig:aux} shows that the need for $\phi_{aux}$ depends on the dataset.
In PROTEINS and FRANKENSTEIN, consideration of $\phi_{aux}$, i.e., $\lambda \geq 0.2$, negatively affects the performance.
To further support our observation, we investigate the closeness centrality of graphs in the datasets (See Supplementary Figure \hyperref[fig:closeness]{1} and Appendix \ref{appendix:closness}).
Interestingly, the graphs from PROTEINS and FRANKENSTEIN have a larger average of closeness centrality than the other datasets.
Note that a node with a higher closeness centrality has the shortest distance to all other nodes, so the node's information propagates more easily throughout the graph.
For this reason, $\phi_{aux}$ that considers information of local neighborhood seems to degrade the performance.
In summary, the effects of $\phi_{aux}$ and the regularization $\lambda$ may depend on graph properties such as closeness centrality.

\begin{figure*}[t]
\centering
\includegraphics[width=\textwidth]{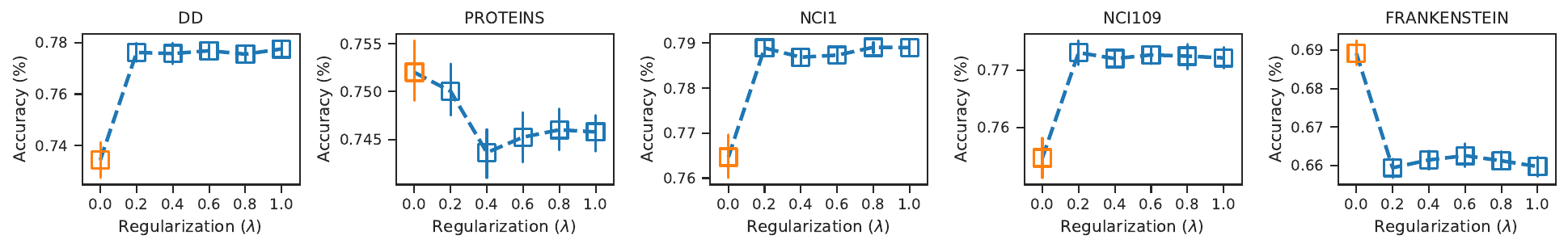} 
\caption{Performance of SPGP on the TUDataset by varying the regularization parameter $(\lambda)$. Average accuracy and standard deviation for 20 random seeds are shown as box plots. The orange box plot denotes $\lambda = 0.0$. The blue box plots denote $\lambda \geq 0.2$.}
\label{fig:aux}
\end{figure*}

\begin{figure*}[!t]
\centering
\includegraphics[width=\textwidth]{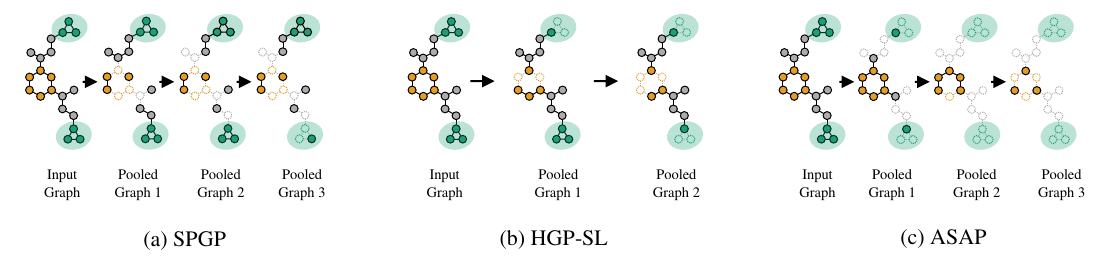} 
\caption{Visualization of different pooling methods. The input graph is sampled from FRANKENSTEIN dataset. Orange nodes represent the nodes belong to a biconnected component. Green nodes represent the nodes belong to both a biconnected component and a clique. We omit newly added edges in HGP-SL and ASAP for better visibility.}
\label{fig:pool}
\end{figure*}

\subsection{Visualization: What Structure Does SPGP Preserve? (Q3)} \label{sec:q3}
We visualize the effects of different pooling strategies of SPGP, HGP-SL, and ASAP.
We randomly sample a graph, which has mutagenicity, from FRANKENSTEIN dataset.
As shown in Figure \ref{fig:pool}, the molecule contains one aromatic ring (orange) and two functional groups of three-membered heterocycles (green) such as oxirane or aziridine. 
These two reactive groups are known as toxicophores, which are closely related to the mutagenic activity of the chemical compounds~\cite{kazius2005derivation}.
Figure \ref{fig:pool}(a) shows that SPGP well preserves both of the  informative structures after pooling. 
Specifically, since the three-membered heterocycles are cliques as well as BCCs, SPGP potentially learns the importance of the structures through the guidance of structure prototypes.
On the other hand, in Figure \ref{fig:pool}(b) and \ref{fig:pool}(c), HGP-SL and ASAP relatively focus on carbon chains and an aromatic ring, respectively.

\begin{table}[t]
\caption{Performances on OGB dataset. Evaluation metrics for ogbg-molhiv and ogbg-molpcba are ROC-AUC (\%) and average precision (\%), respectively. Average metric and standard deviations for 20 random seeds are reported. The best result and its comparable results ($p > 0.05$) from t-test are shown in bold. GCN denotes reported results in the leader-board. GCN$^\dagger$ denotes reproduced results in the 20 random seeds. OOM denotes out-of-memory error.}
\label{tab:ogb}
\centering
\begin{tabular}{ccc}
\toprule
  Methods   & ogbg-molhiv & ogbg-molpcba \\ \midrule
GCN & $76.06 \pm 0.97$ & $20.20 \pm 0.24$ \\ 
GCN$^\dagger$    & $75.86 \pm 1.15$       & $20.18 \pm 0.22$    \\ \midrule
HGP-SL & $72.46 \pm 1.55$     & $18.64 \pm 0.28$      \\
ASAP   & $76.89  \pm 1.90$     & OOM          \\ 
SPGP    & $\boldsymbol{78.15 \pm 0.83}$   & $\boldsymbol{23.95  \pm 0.97}$    \\ \bottomrule
\end{tabular}
\end{table}

\subsection{Scalability on large-Scale OGB datasets (Q4)} \label{sec:q4}

We compare SPGP to two best baseline models, ASAP and HGP-SL, on ogbg-molhiv and ogbg-molpcba from the large-scale OGB dataset.
For fair comparisons, we experiment with SPGP, ASAP, and HGP-SL in the same setting (See Appendix \ref{appendix:training} for details).
In both datasets, SPGP outperforms all baselines including the simple GCN model~\cite{kipf2016semi} without pooling operations.
Rather, unlike HGP-SL, which degrades performances through pooling operations, SPGP can extract meaningful structures on the large-scale datasets.
For ASAP, which is the best competitor in TUDataset, it shows performance improvements through pooling on ogbg-molhiv, but with out-of-memory errors on the larger dataset, ogbg-molpcba.
Due to the efficiency of SPGP, it can be jointly used with advanced graph convolution layers and the results are discussed in Appendix \ref{appendix:vn_res}

\begin{wrapfigure}{r}{0.55\textwidth}
\vspace{-13pt}
\centering
\includegraphics[width=0.53\textwidth]{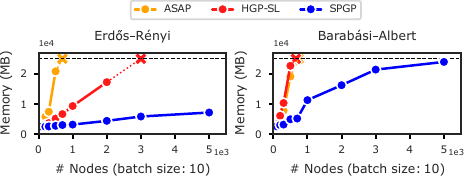} 
\vspace{-5pt}
\caption{Memory footprint of SPGP compared with baselines. Memory usage is measured for graphs of various sizes generated from two random graph models. Black dashed line represents maximum memory of GPU. X denotes out-of-memory error.}
\vspace{-13pt}
\label{fig:mem_test}
\end{wrapfigure}

To support this observation, we evaluate the GPU memory footprint of SPGP with the baselines on the graphs generated by two random graph models: Erdős–Rényi~\cite{gilbert1959random} and Barabási–Albert~\cite{barabasi1999emergence} (See Appendix \ref{appendix:mem_test} for details).
Figure \ref{fig:mem_test} shows that SPGP is very effective and scalable compared to other methods in terms of memory usage.
This is, as shown in Section \ref{sec:complexity}, because ASAP and HGP-SL additionally learn the structure of the pooled graph to extract meaningful nodes, whereas SPGP learns global structural information with fewer operations through structure prototype vectors.

\section{Conclusions} \label{sec:conclusion}


\paragraph{Limitations}
There are potential limitations of SPGP that would be an important direction for future work.
(i) Isolated nodes that do not belong to the structure prototypes are hard to address by SPGP.
(ii) Given prior graph structures, SPGP is limited to utilizing information such as size and specific shape based on only node membership information.
Other structural information may be useful.
(iii) Learning the interactions between different types of structure prototypes remains a challenge.

\paragraph{Summary}
We propose a novel graph pooling method, SPGP, for graph-level classification.
The key to success of SPGP is the consideration of informative global graph structures to obtain accurate graph-level representations.
The structure prototypes guide the node scoring scheme during pooling.
Experiments on multiple datasets and ablation studies show that utilization of the prior graph structures achieves better performance than existing pooling methods.
In this study, two well-known structures in graph theory, BCC and cliques are utilized as the prior graph structures.
In future research, SPGP can be easily extended to utilize other graph structures inspired by domain knowledge, e.g., network biomarkers for interpreting diseases and chemical building blocks in molecular graph generation.

\section*{Acknowledgements}
This research was supported by Basic Science Research Program through the National Research Foundation of Korea (NRF) funded by the Ministry of Education (NRF-2021R1A6A3A01086898) and by Institute of Information \& communications Technology Planning \& Evaluation (IITP) grant funded by the Korea government(MSIT) [NO.2021-0-01343, Artificial Intelligence Graduate School Program (Seoul National University)]

\bibliographystyle{plainnat}
\bibliography{References}

\begin{thebibliography}{51}
\providecommand{\natexlab}[1]{#1}
\providecommand{\url}[1]{\texttt{#1}}
\expandafter\ifx\csname urlstyle\endcsname\relax
  \providecommand{\doi}[1]{doi: #1}\else
  \providecommand{\doi}{doi: \begingroup \urlstyle{rm}\Url}\fi

\bibitem[Barab{\'a}si and Albert(1999)]{barabasi1999emergence}
Albert-L{\'a}szl{\'o} Barab{\'a}si and R{\'e}ka Albert.
\newblock Emergence of scaling in random networks.
\newblock \emph{science}, 286\penalty0 (5439):\penalty0 509--512, 1999.

\bibitem[Borgwardt et~al.(2005)Borgwardt, Ong, Sch{\"o}nauer, Vishwanathan,
  Smola, and Kriegel]{borgwardt2005protein}
Karsten~M Borgwardt, Cheng~Soon Ong, Stefan Sch{\"o}nauer, SVN Vishwanathan,
  Alex~J Smola, and Hans-Peter Kriegel.
\newblock Protein function prediction via graph kernels.
\newblock \emph{Bioinformatics}, 21\penalty0 (suppl\_1):\penalty0 i47--i56,
  2005.

\bibitem[Broido and Clauset(2019)]{broido2019scale}
Anna~D Broido and Aaron Clauset.
\newblock Scale-free networks are rare.
\newblock \emph{Nature communications}, 10\penalty0 (1):\penalty0 1--10, 2019.

\bibitem[Council et~al.(1995)]{national1995mathematical}
National~Research Council et~al.
\newblock \emph{Mathematical challenges from theoretical/computational
  chemistry}.
\newblock National Academies Press, 1995.

\bibitem[Defferrard et~al.(2016)Defferrard, Bresson, and
  Vandergheynst]{defferrard2016convolutional}
Micha{\"e}l Defferrard, Xavier Bresson, and Pierre Vandergheynst.
\newblock Convolutional neural networks on graphs with fast localized spectral
  filtering.
\newblock \emph{Advances in neural information processing systems},
  29:\penalty0 3844--3852, 2016.

\bibitem[Dobson and Doig(2003)]{dobson2003distinguishing}
Paul~D Dobson and Andrew~J Doig.
\newblock Distinguishing enzyme structures from non-enzymes without alignments.
\newblock \emph{Journal of molecular biology}, 330\penalty0 (4):\penalty0
  771--783, 2003.

\bibitem[Evtushenko and Kleinberg(2021)]{evtushenko2021paradox}
Anna Evtushenko and Jon Kleinberg.
\newblock The paradox of second-order homophily in networks.
\newblock \emph{Scientific Reports}, 11\penalty0 (1):\penalty0 1--10, 2021.

\bibitem[Fey and Lenssen(2019)]{Fey/Lenssen/2019}
Matthias Fey and Jan~E. Lenssen.
\newblock Fast graph representation learning with {PyTorch Geometric}.
\newblock In \emph{ICLR Workshop on Representation Learning on Graphs and
  Manifolds}, 2019.

\bibitem[Gao and Ji(2019)]{gao2019graph}
Hongyang Gao and Shuiwang Ji.
\newblock Graph u-nets.
\newblock In \emph{International Conference on Machine Learning}, pages
  2083--2092. PMLR, 2019.

\bibitem[Gao et~al.(2021)Gao, Dai, Li, Xiong, and Frossard]{gao2021ipool}
Xing Gao, Wenrui Dai, Chenglin Li, Hongkai Xiong, and Pascal Frossard.
\newblock ipool--information-based pooling in hierarchical graph neural
  networks.
\newblock \emph{IEEE Transactions on Neural Networks and Learning Systems},
  2021.

\bibitem[Garg et~al.(2020)Garg, Jegelka, and Jaakkola]{garg2020generalization}
Vikas Garg, Stefanie Jegelka, and Tommi Jaakkola.
\newblock Generalization and representational limits of graph neural networks.
\newblock In \emph{International Conference on Machine Learning}, pages
  3419--3430. PMLR, 2020.

\bibitem[Gilbert(1959)]{gilbert1959random}
Edgar~N Gilbert.
\newblock Random graphs.
\newblock \emph{The Annals of Mathematical Statistics}, 30\penalty0
  (4):\penalty0 1141--1144, 1959.

\bibitem[Grover and Leskovec(2016)]{grover2016node2vec}
Aditya Grover and Jure Leskovec.
\newblock node2vec: Scalable feature learning for networks.
\newblock In \emph{Proceedings of the 22nd ACM SIGKDD international conference
  on Knowledge discovery and data mining}, pages 855--864, 2016.

\bibitem[Hamilton et~al.(2017)Hamilton, Ying, and
  Leskovec]{hamilton2017inductive}
William~L Hamilton, Rex Ying, and Jure Leskovec.
\newblock Inductive representation learning on large graphs.
\newblock In \emph{Proceedings of the 31st International Conference on Neural
  Information Processing Systems}, pages 1025--1035, 2017.

\bibitem[Hochbaum(1993)]{hochbaum1993should}
Dorit~S Hochbaum.
\newblock Why should biconnected components be identified first.
\newblock \emph{Discrete applied mathematics}, 42\penalty0 (2-3):\penalty0
  203--210, 1993.

\bibitem[Hong et~al.(2018)Hong, Nguyen, Meidiana, Li, and Eades]{hong2018bc}
Seok-Hee Hong, Quan Nguyen, Amyra Meidiana, Jiaxi Li, and Peter Eades.
\newblock Bc tree-based proxy graphs for visualization of big graphs.
\newblock In \emph{2018 IEEE Pacific Visualization Symposium (PacificVis)},
  pages 11--20. IEEE, 2018.

\bibitem[Hopcroft and Tarjan(1973)]{hopcroft1973algorithm}
John Hopcroft and Robert Tarjan.
\newblock Algorithm 447: efficient algorithms for graph manipulation.
\newblock \emph{Communications of the ACM}, 16\penalty0 (6):\penalty0 372--378,
  1973.

\bibitem[Hu et~al.(2020{\natexlab{a}})Hu, Liu, Gomes, Zitnik, Liang, Pande, and
  Leskovec]{hu2020strategies}
W~Hu, B~Liu, J~Gomes, M~Zitnik, P~Liang, V~Pande, and J~Leskovec.
\newblock Strategies for pre-training graph neural networks.
\newblock In \emph{International Conference on Learning Representations
  (ICLR)}, 2020{\natexlab{a}}.

\bibitem[Hu et~al.(2020{\natexlab{b}})Hu, Fey, Zitnik, Dong, Ren, Liu, Catasta,
  and Leskovec]{hu2020ogb}
Weihua Hu, Matthias Fey, Marinka Zitnik, Yuxiao Dong, Hongyu Ren, Bowen Liu,
  Michele Catasta, and Jure Leskovec.
\newblock Open graph benchmark: Datasets for machine learning on graphs.
\newblock \emph{arXiv preprint arXiv:2005.00687}, 2020{\natexlab{b}}.

\bibitem[Kazius et~al.(2005)Kazius, McGuire, and Bursi]{kazius2005derivation}
Jeroen Kazius, Ross McGuire, and Roberta Bursi.
\newblock Derivation and validation of toxicophores for mutagenicity
  prediction.
\newblock \emph{Journal of medicinal chemistry}, 48\penalty0 (1):\penalty0
  312--320, 2005.

\bibitem[Kingma and Ba(2014)]{kingma2014adam}
Diederik~P Kingma and Jimmy Ba.
\newblock Adam: A method for stochastic optimization.
\newblock \emph{arXiv preprint arXiv:1412.6980}, 2014.

\bibitem[Kipf and Welling(2016)]{kipf2016semi}
Thomas~N Kipf and Max Welling.
\newblock Semi-supervised classification with graph convolutional networks.
\newblock \emph{arXiv preprint arXiv:1609.02907}, 2016.

\bibitem[Lee et~al.(2019)Lee, Lee, and Kang]{lee2019self}
Junhyun Lee, Inyeop Lee, and Jaewoo Kang.
\newblock Self-attention graph pooling.
\newblock In \emph{International Conference on Machine Learning}, pages
  3734--3743. PMLR, 2019.

\bibitem[Li et~al.(2016)Li, Tarlow, Brockschmidt, and Zemel]{li2015gated}
Yujia Li, Daniel Tarlow, Marc Brockschmidt, and Richard Zemel.
\newblock Gated graph sequence neural networks.
\newblock In \emph{International Conference on Learning Representations
  (ICLR)}, 2016.

\bibitem[Ma et~al.(2019)Ma, Wang, Aggarwal, and
  Tang]{Ma:2019:GCN:3292500.3330982}
Yao Ma, Suhang Wang, Charu~C. Aggarwal, and Jiliang Tang.
\newblock Graph convolutional networks with eigenpooling.
\newblock In \emph{Proceedings of the 25th ACM SIGKDD International Conference
  on Knowledge Discovery \& Data Mining}, pages 723--731, 2019.

\bibitem[McPherson et~al.(2001)McPherson, Smith-Lovin, and
  Cook]{mcpherson2001birds}
Miller McPherson, Lynn Smith-Lovin, and James~M Cook.
\newblock Birds of a feather: Homophily in social networks.
\newblock \emph{Annual review of sociology}, 27\penalty0 (1):\penalty0
  415--444, 2001.

\bibitem[Morris et~al.(2020)Morris, Kriege, Bause, Kersting, Mutzel, and
  Neumann]{Morris+2020}
Christopher Morris, Nils~M. Kriege, Franka Bause, Kristian Kersting, Petra
  Mutzel, and Marion Neumann.
\newblock Tudataset: A collection of benchmark datasets for learning with
  graphs.
\newblock In \emph{ICML 2020 Workshop on Graph Representation Learning and
  Beyond (GRL+ 2020)}, 2020.
\newblock URL \url{www.graphlearning.io}.

\bibitem[Nikolentzos et~al.(2020)Nikolentzos, Dasoulas, and
  Vazirgiannis]{nikolentzos2020k}
Giannis Nikolentzos, George Dasoulas, and Michalis Vazirgiannis.
\newblock k-hop graph neural networks.
\newblock \emph{Neural Networks}, 130:\penalty0 195--205, 2020.

\bibitem[Oono and Suzuki(2019)]{oono2019graph}
Kenta Oono and Taiji Suzuki.
\newblock Graph neural networks exponentially lose expressive power for node
  classification.
\newblock In \emph{International Conference on Learning Representations}, 2019.

\bibitem[Orsini et~al.(2015)Orsini, Frasconi, and De~Raedt]{orsini2015graph}
Francesco Orsini, Paolo Frasconi, and Luc De~Raedt.
\newblock Graph invariant kernels.
\newblock In \emph{Twenty-Fourth International Joint Conference on Artificial
  Intelligence}, 2015.

\bibitem[Otter et~al.(2017)Otter, Porter, Tillmann, Grindrod, and
  Harrington]{otter2017roadmap}
Nina Otter, Mason~A Porter, Ulrike Tillmann, Peter Grindrod, and Heather~A
  Harrington.
\newblock A roadmap for the computation of persistent homology.
\newblock \emph{EPJ Data Science}, 6:\penalty0 1--38, 2017.

\bibitem[Ou et~al.(2016)Ou, Cui, Pei, Zhang, and Zhu]{ou2016asymmetric}
Mingdong Ou, Peng Cui, Jian Pei, Ziwei Zhang, and Wenwu Zhu.
\newblock Asymmetric transitivity preserving graph embedding.
\newblock In \emph{Proceedings of the 22nd ACM SIGKDD international conference
  on Knowledge discovery and data mining}, pages 1105--1114, 2016.

\bibitem[Petri et~al.(2013)Petri, Scolamiero, Donato, and
  Vaccarino]{petri2013topological}
Giovanni Petri, Martina Scolamiero, Irene Donato, and Francesco Vaccarino.
\newblock Topological strata of weighted complex networks.
\newblock \emph{PloS one}, 8\penalty0 (6):\penalty0 e66506, 2013.

\bibitem[Ranjan et~al.(2020)Ranjan, Sanyal, and Talukdar]{ranjan2020asap}
Ekagra Ranjan, Soumya Sanyal, and Partha Talukdar.
\newblock Asap: Adaptive structure aware pooling for learning hierarchical
  graph representations.
\newblock In \emph{Proceedings of the AAAI Conference on Artificial
  Intelligence}, volume~34, pages 5470--5477, 2020.

\bibitem[Shao et~al.(2020)Shao, Guo, Gu, Wang, Li, and Yu]{shao2020efficient}
Zhenzhen Shao, Na~Guo, Yu~Gu, Zhigang Wang, Fangfang Li, and Ge~Yu.
\newblock Efficient closeness centrality computation for dynamic graphs.
\newblock In \emph{International Conference on Database Systems for Advanced
  Applications}, pages 534--550. Springer, 2020.

\bibitem[Srihari et~al.(2017)Srihari, Yong, and Wong]{srihari2017computational}
Sriganesh Srihari, Chern~Han Yong, and Limsoon Wong.
\newblock \emph{Computational prediction of protein complexes from protein
  interaction networks}.
\newblock Morgan \& Claypool, 2017.

\bibitem[Veli{\v{c}}kovi{\'c} et~al.(2018)Veli{\v{c}}kovi{\'c}, Cucurull,
  Casanova, Romero, Li{\`o}, and Bengio]{velivckovic2018graph}
Petar Veli{\v{c}}kovi{\'c}, Guillem Cucurull, Arantxa Casanova, Adriana Romero,
  Pietro Li{\`o}, and Yoshua Bengio.
\newblock Graph attention networks.
\newblock In \emph{International Conference on Learning Representations}, 2018.

\bibitem[Vinyals et~al.(2015)Vinyals, Bengio, and Kudlur]{vinyals2015order}
Oriol Vinyals, Samy Bengio, and Manjunath Kudlur.
\newblock Order matters: Sequence to sequence for sets.
\newblock In \emph{International Conference on Learning Representations
  (ICLR)}, 2015.

\bibitem[Wale and Karypis(2006)]{wale2006acyclic}
Nikil Wale and George Karypis.
\newblock Acyclic subgraph based descriptor spaces for chemical compound
  retrieval and classification.
\newblock Technical report, MINNESOTA UNIV MINNEAPOLIS DEPT OF COMPUTER
  SCIENCE, 2006.

\bibitem[Wu et~al.(2018)Wu, Ramsundar, Feinberg, Gomes, Geniesse, Pappu,
  Leswing, and Pande]{wu2018moleculenet}
Zhenqin Wu, Bharath Ramsundar, Evan~N Feinberg, Joseph Gomes, Caleb Geniesse,
  Aneesh~S Pappu, Karl Leswing, and Vijay Pande.
\newblock Moleculenet: a benchmark for molecular machine learning.
\newblock \emph{Chemical science}, 9\penalty0 (2):\penalty0 513--530, 2018.

\bibitem[Xu et~al.(2018{\natexlab{a}})Xu, Hu, Leskovec, and
  Jegelka]{xu2018powerful}
Keyulu Xu, Weihua Hu, Jure Leskovec, and Stefanie Jegelka.
\newblock How powerful are graph neural networks?
\newblock In \emph{International Conference on Learning Representations},
  2018{\natexlab{a}}.

\bibitem[Xu et~al.(2018{\natexlab{b}})Xu, Li, Tian, Sonobe, Kawarabayashi, and
  Jegelka]{xu2018representation}
Keyulu Xu, Chengtao Li, Yonglong Tian, Tomohiro Sonobe, Ken-ichi Kawarabayashi,
  and Stefanie Jegelka.
\newblock Representation learning on graphs with jumping knowledge networks.
\newblock In \emph{International Conference on Machine Learning}, pages
  5453--5462. PMLR, 2018{\natexlab{b}}.

\bibitem[Ying et~al.(2018)Ying, You, Morris, Ren, Hamilton, and
  Leskovec]{ying2018hierarchical}
Rex Ying, Jiaxuan You, Christopher Morris, Xiang Ren, William~L Hamilton, and
  Jure Leskovec.
\newblock Hierarchical graph representation learning with differentiable
  pooling.
\newblock In \emph{Proceedings of the 32nd International Conference on Neural
  Information Processing Systems}, pages 4805--4815, 2018.

\bibitem[You et~al.(2019)You, Ying, and Leskovec]{you2019position}
Jiaxuan You, Rex Ying, and Jure Leskovec.
\newblock Position-aware graph neural networks.
\newblock In \emph{International Conference on Machine Learning}, pages
  7134--7143. PMLR, 2019.

\bibitem[Yuan and Ji(2020)]{yuan2020structpool}
Hao Yuan and Shuiwang Ji.
\newblock Structpool: Structured graph pooling via conditional random fields.
\newblock In \emph{Proceedings of the 8th International Conference on Learning
  Representations}, 2020.

\bibitem[Yun et~al.(2021)Yun, Kim, Lee, Kang, and Kim]{yun2021neo}
Seongjun Yun, Seoyoon Kim, Junhyun Lee, Jaewoo Kang, and Hyunwoo~J Kim.
\newblock Neo-gnns: Neighborhood overlap-aware graph neural networks for link
  prediction.
\newblock \emph{Advances in Neural Information Processing Systems}, 34, 2021.

\bibitem[Zhang et~al.(2020)Zhang, Wang, Li, Zhu, Shen, Li, Lu, Shah, and
  Bennamoun]{zhang2020structure}
Liang Zhang, Xudong Wang, Hongsheng Li, Guangming Zhu, Peiyi Shen, Ping Li,
  Xiaoyuan Lu, Syed Afaq~Ali Shah, and Mohammed Bennamoun.
\newblock Structure-feature based graph self-adaptive pooling.
\newblock In \emph{Proceedings of The Web Conference 2020}, pages 3098--3104,
  2020.

\bibitem[Zhang and Chen(2018)]{zhang2018link}
Muhan Zhang and Yixin Chen.
\newblock Link prediction based on graph neural networks.
\newblock \emph{Advances in neural information processing systems}, 31, 2018.

\bibitem[Zhang et~al.(2018)Zhang, Cui, Neumann, and Chen]{zhang2018end}
Muhan Zhang, Zhicheng Cui, Marion Neumann, and Yixin Chen.
\newblock An end-to-end deep learning architecture for graph classification.
\newblock In \emph{Thirty-Second AAAI Conference on Artificial Intelligence},
  2018.

\bibitem[Zhang et~al.(2005)Zhang, Abu-Khzam, Baldwin, Chesler, Langston, and
  Samatova]{zhang2005genome}
Yun Zhang, Faisal~N Abu-Khzam, Nicole~E Baldwin, Elissa~J Chesler, Michael~A
  Langston, and Nagiza~F Samatova.
\newblock Genome-scale computational approaches to memory-intensive
  applications in systems biology.
\newblock In \emph{SC'05: Proceedings of the 2005 ACM/IEEE Conference on
  Supercomputing}, pages 12--12. IEEE, 2005.

\bibitem[Zhang et~al.(2019)Zhang, Bu, Ester, Zhang, Yao, Yu, and
  Wang]{zhang2019hierarchical}
Zhen Zhang, Jiajun Bu, Martin Ester, Jianfeng Zhang, Chengwei Yao, Zhi Yu, and
  Can Wang.
\newblock Hierarchical graph pooling with structure learning.
\newblock \emph{arXiv preprint arXiv:1911.05954}, 2019.

\end{thebibliography}

\newpage

\appendix

\section{Algorithm}

\subsection{Overall Structure of SPGP}\label{appendix:pseudo_code}

The overall structure of SPGP is described in Algorithm 1.

\begin{algorithm}[h]
\caption{Overall Structure of SPGP}
\label{alg:overall}
\textbf{Input}: Graph $\mathcal{G} = (\mathcal{V}, \mathcal{E})$; Adjacency matrix $A$; Node feature matrix $X$;  Pooling operation SPGP; Structure Prototype $SP$; Number of layers $L$, readout function Readout\\
\textbf{Output}: A predicted graph label $\widehat{y}$ 
\begin{algorithmic}[1] 
\STATE $X_{aug}$ $\leftarrow$ $FA(X, SP)$  // Add global structure related features (Section A.2)
\STATE $H^1$ $\leftarrow$ LayerNorm(GCN($X_{aug}$, $A$))
\STATE $H^1, A^1$ $\leftarrow$ SPGP($H^1$, $A$, $SP$)
\STATE $HS$ $\leftarrow$ Readout($H^1$)
\FOR{i=2...L}
\STATE $H^{i+1}$ $\leftarrow$ LayerNorm(GCN($H^{i}$, $A^{i}$))
\STATE $H^{i+1}, A^{i+1}$ $\leftarrow$ SPGP($H^{i}$, $A^{i}$, $SP$)
\STATE $HS$ $\leftarrow$ $HS$ $+$ Readout($H^{i+1}$)
\ENDFOR
\STATE $\textrm{out}$ $\leftarrow$ MLP($HS$) // Multi-Layer Perceptron
\STATE $\widehat{y}$ $\leftarrow$ $argmax(\textrm{out})$
\STATE \textbf{return} $\widehat{y}$
\end{algorithmic}
\end{algorithm}


\begin{algorithm}[h]
\caption{Graph Pooling of SPGP$(H^{(k)}, A^{(k)}, SP, p)$}
\label{alg:pooling}
\begin{algorithmic}[1] 
\STATE $\widetilde{H}^{(k)}$ $\leftarrow$ $H^{(k)} + \textnormal{LeakyReLU}\left(\boldsymbol{W}^{(k)}\textnormal{Concat}(T^{(k)}, C^{(k)})\right)$ according to Equation 1
\STATE Calculate $\phi_{prototype}$ according to Equation 5
\STATE Calculate $\phi_{aux}$ according to Equation 7
\STATE $\phi(v)$ $\leftarrow$ $\sigma\left(\phi_{prototype}(v) + \lambda \, \phi_{aux}(v)\right)$ 
\STATE $\textrm{idx}$ $\leftarrow$ $\textrm{Top-K}(\phi(v^{(k)}), p)$ 
\STATE $H^{(k+1)}$ $\leftarrow$ $H^{(k)}\:[\textrm{idx}, \;:]  \odot \phi(.)$ 
\STATE $A^{(k+1)}$ $\leftarrow$ $ A^{(k)}\:[\textrm{idx}, \;\textrm{idx}]$ 
\STATE \textbf{return} $H^{(k+1)}$, $A^{(k+1)}$ 
\end{algorithmic}
\end{algorithm}

\subsection{Node Representations with Contextual Information}\label{appendix:context_info}
In a graph, the structural information of a node can be enriched with contextual information of the node.
Inspired by previous works on message passing network~\cite{hu2020strategies}, we utilize additional graph convolution networks~\cite{kipf2016semi} to obtain node representations with contextual information.
A contextual node embedding $\widetilde{h}^{(k)}_v$ at layer $k$ associated with node $v$ can be formulated as:
\begin{equation}
    \widetilde{h}^{(k)}_v = h^{(k)}_v + \textnormal{LeakyReLU}\left(\boldsymbol{W}^{(k)}\textnormal{Concat}(t^{(k)}_v, c^{(k)}_v)\right)
    \label{eq:node_context_score}
\end{equation}
where $h^{(k)}_v \in H^{(k)}$ and $\boldsymbol{W}^{(k)} \in \mathbb{R}^{2d \times d}$ (a bias is also used but omitted for clarity).
The target node representation $t^{(k)}_v$ and the its context node representation $c^{(k)}_v$ can be expressed as:
\begin{align}
    t^{(k)}_v &= \sum_{u \in \widehat{\mathcal{N}}^{(k)}_t(v)} \frac{e_{u,v}}{\sqrt{|\widehat{\mathcal{N}}^{(k)}_t(u)||\widehat{\mathcal{N}}^{(k)}_t(v)|}}\boldsymbol{W}^{(k)}_t h^{(k)}_u\\
    c^{(k)}_v &= \sum_{u \in \widehat{\mathcal{N}}^{(k)}_c(v)} \frac{e_{u,v}}{\sqrt{|\widehat{\mathcal{N}}^{(k)}_c(u)||\widehat{\mathcal{N}}^{(k)}_c(v)|}}\boldsymbol{W}^{(k)}_c h^{(k)}_u
\end{align}
where $\widehat{\mathcal{N}}^{(k)}_t(v)$ and $\widehat{\mathcal{N}}^{(k)}_c(v)$ denote the direct neighborhood and the context nodes of the node $v$, respectively, in the self-looped adjacency matrix $\widehat{A}^{(k)} = A^{(k)} + I^{(k)}$.
The context nodes can be any k-hop ($k\geq2$) neighborhood in the graph.
In this paper, since global information of the graph is addressed through the structure prototypes, we simply employ the two-hop neighborhood.
$e_{u,v}$ denotes edge weight between node $u$ and node $v$ if available, otherwise 1.
$\boldsymbol{W}^{(k)}_t,\boldsymbol{W}^{(k)}_c \in \mathbb{R}^{d \times d}$ are learnable parameters.

\subsection{Augmentation of structure prototype related features} 
As another aspect of utilizing structure prototypes, we augment node features from the graph itself.
The node features enhanced by the structure prototypes are expected to help in the pooling operation.
Given a node $v$ of a graph $\mathcal{G}_i$, the augmented node features are defined as:
\begin{itemize}
    \item Size of biconnected components $\delta_1 : \max_{p \in N_{v}^{bcc}}\frac{|\mathcal{S}^{BCC}_{i,p}|}{|\mathcal{S}^{BCC}_{i,max}|}$
    \item Size of cliques $\delta_2 : \max_{q \in N_{v}^{cq}}\frac{|\mathcal{S}^{CQ}_{i,q}|}{|\mathcal{S}^{CQ}_{i,max}|}$
    \item Number of cliques $\delta_3 : \frac{|N_{v}^{cq}|}{|\mathcal{S}^{CQ}_i|}$
\end{itemize}
where $\mathcal{S}^{BCC}_{i, max}$ is the largest biconnected components in the graph $\mathcal{G}_i$.
$\mathcal{S}^{CQ}_{i, max}$ is the maximum cliques in the graph $\mathcal{G}_i$.
$N^{bcc}_v$ and $N^{cq}_v$ are biconnected components and cliques containing node $v$, respectively.
This augmentation can be easily extended to other prior graph structures.

\subsection{Computational Complexity} \label{appendix:complexity}
For a single layer of vanilla GCN, computational complexity is $O(md + nd^2)$. 
SPGP consists of four steps: (1) contextual node embedding, (2) prototype score, (3) auxiliary score, (4) topk pooling. 
Assume that $\mathcal{G}=(\mathcal{V}, \mathcal{E}, \mathcal{X})$, $|\mathcal{V}|=n$, $|\mathcal{E}|=m$, $\mathcal{X} \in \mathbb{R}^{n \times d}$ at layer $l$, computational complexity of each step is as follow, 
\begin{enumerate}
    \item Node Presentation with Contextual Information - Equation \ref{eq:node_context_score}: We calculates the contextual node embedding with right-to-left multiplication. Since adjacency matrix $A$ is stored as a sparse matrix, the computational complexity is $O( md + nd^2 )$. It matches the computational complexity of the vanilla GCN.

    \item Prototype Score - Equation \ref{eq:prototype_score}: The prototype score consists of two terms: relation module “$q$” and node’s self-information. The computational complexity is $O(sdn + dn)$, where $s$ is the number of prototypes and $s, d << n$. 
    \item Auxiliary Score - Equation \ref{eq:aux_score}: This procedure requires aggregation of neighborhood information. So, computational complexity is $O( md + nd^2 )$.

    \item Top-k pooling - Equation \ref{eq:idx} and \ref{eq:filter}: This procedure filters the node feature matrix and graph structure by selected nodes. It takes $O(n+m)$.

\end{enumerate}
In summary, the computational complexity of SPGP is $O(md + nd^2 + sdn)$. With $L$ layers graph neural network with GCN and SPGP with pooling ratio $p$, total computational complexity is $O(\frac{1-p^L}{1-p}(md+nd^2+sdn)) \approx O(L(md+nd^2+sdn)$.

\section{Proof}

\subsection{Proof for Theorem 1} \label{appendix:proof}
We assume that all nodes have the same feature vector, which means that the only information we can learn from the graph is structural information.
We first prove that the standard GNN generates the same embedding for every node in a regular graph (Lemma \ref{lemma1}). 
After that, we show that the SPGP can learn the structure role of nodes in a regular graph, but not the standard k-hop neighborhood based graph pooling methods (Theorem \hyperref[appendix:theorem1]{1}).


\begin{lemma}
The standard GNN generates the same embedding for every node in a regular graph.
\label{lemma1}
\end{lemma}

\begin{proof}
The standard GNN updates node representations by following equations.
\begin{equation*}
    \begin{split}
    a^{(k)}_v &= \textnormal{AGGREGATE}^{(k)}(\{ h^{(k-1)}_u \mid u \in N(v) \}) \\
    h^{(k)}_v &= \textnormal{MERGE}^{(k)}(h^{(k-1)}_v, a^{(k)}_v)
    \end{split}
\end{equation*}
Given a regular graph $\mathcal{G} = (\mathcal{V}, \mathcal{E})$, we show for an arbitrary iteration $k$ and nodes $v_1, v_2 \in \mathcal{V}, v_1 \neq v_2$ that $h^{(k)}_{v_1} = h^{(k)}_{v_2}$.
In iteration $k=0$, $h^{(0)}_{v_1} = h^{(0)}_{v_2}$ is hold since all nodes are the same feature vector.
Assume for induction that $h^{(k-1)}_{v_1} = h^{(k-1)}_{v_2}$ for arbitrary nodes $v_1$ and $v_2$.
Let $a^{(k)}_{v_1} = \textnormal{AGGREGATE}^{(k)}(\{ h^{(k-1)}_{u_1} \mid u_1 \in N({v_1}) \})$ and $a^{(k)}_{v_2} = \textnormal{AGGREGATE}^{(k)}(\{ h^{(k-1)}_{u_2} \mid u_2 \in N({v_2}) \})$ be the aggregated messages from neighbors of $v_1$ and $v_2$.
By the induction hypothesis, we know that $\{ h^{(k-1)}_{u_1} \mid u_1 \in N({v_1}) \} = \{ h^{(k-1)}_{u_2} \mid u_2 \in N({v_2}) \}$ and that $a^{(k)}_{v_1} = a^{(k)}_{v_2}$.
Therefore, the input of $\textnormal{MERGE}^{(k)}$ is identical in $v_1$ and $v_2$.
This proves that $h^{(k)}_{v_1} = h^{(k)}_{v_2}$, and there by the lemma.
\end{proof}


\begin{appendix_theorem} \label{appendix:theorem1}
Given a regular graph that has prior graph structures,
\begin{enumerate}[label=(\roman*)]
    \item Standard k-hop neighborhood based graph pooling methods cannot distinguish the nodes.
    \item SPGP can assign different node scores based on the structural role of nodes.
\end{enumerate}

\end{appendix_theorem}

\begin{proof}
For the standard k-hop neighborhood based graph pooling methods, node representations for measuring the node score are based on the standard GNN. 
By Lemma \ref{lemma1}, in the regular graph, all node representations are the same. 
Therefore, the existing pooling methods cannot differentiate between nodes, even if they aggregate k-hop neighbor information, since the amount of information they collect is the same.
This proves the first part of the theorem.
We prove that the second part by showing that Equation \ref{eq:prototype_score} and \ref{eq:prototype} can learn the additional information of nodes belonging to the prior graph structures.
The structural roles contained in the given structure prototypes can be reflected in the nodes through the equations.
Hence, SPGP can assign different node scores based on the given structural information of the structure prototypes.
\end{proof}

\subsection{Experiment on regular graphs with degree perturbations} \label{appendix:regular_test}
To further support Theorem \ref{theorem1}, we measure the diversity of node scores in the regular graphs with degree perturbations.
We generate 100 random 6-regular graphs with 100 nodes.
Then, for each graph, it goes through the process of perturbing the degree of nodes by changing the existing edge to another non-existent edge, i.e., edge rewiring.
That is, starting with the regular graph, the degree distribution on the graph is diversified to give a difference in the amount of information around the k-hop neighborhood.
As shown in Figure \ref{fig:regular_graph}, in the regular graphs, only SPGP can capture node diversity. 
In addition, as the degree of perturbation progresses, the existing methods such as ASAP and HGP-SL also have distinguishing power of nodes, but are not as good as SPGP.

\section{Experimental Details}

\subsection{Datasets Details} \label{appendix:dataset_detail}
Summary statistics for the datasets are shown in Supplementary Table \hyperref[tab:dataset]{1}.

\begin{table}[t]
\caption*{Supplementary Table1: Summary statistics of datasets. $\mathcal{V}_{avg}, \mathcal{E}_{avg}$ are the average number of nodes and edges per graph, respectively.}
\centering
\begin{tabular}{clrrrrr}
\toprule
& Dataset & $\# Graphs$ & $\# Classes$ &$\mathcal{V}_{avg}$ & $\mathcal{E}_{avg}$ & $\# Tasks$ \\ \midrule
\multirow{5}{*}{TUDataset} & DD & 1,178 & 2&  284.3 & 715.7 & 1 \\
&PROTEINS & 1,113 & 2 & 39.1 & 72.8 & 1 \\
&NCI1 & 4,110 & 2 & 29.9 & 32.3 & 1 \\
&NCI109 & 4,127 & 2  & 29.7 & 32.1 & 1 \\
&FRANKENSTEIN & 4,337 & 2 & 16.9 & 17.9 & 1 \\ \midrule
\multirow{2}{*}{OGB Dataset} & ogbg-molhiv & 41,127 & 2  & 25.5 & 27.5 & 1 \\
&ogbg-molpcba & 437,929	 & 2 & 26.0 & 28.1 & 128 \\ \bottomrule
\end{tabular}

\label{tab:dataset}
\end{table}

\begin{table}[t]
\caption*{Supplementary Table 2: Preprocessing time for SPGP.}
\centering
\begin{tabular}{clrrr}
\toprule
& Dataset & $\# Graphs$ & Total Time (s) & Time per graph (s) \\ \midrule
\multirow{5}{*}{TUDataset} & DD & 1,178 & 97 & 0.0823 \\
& PROTEINS & 1,113 & 9 & 0.0081 \\
& NCI1 & 4,110 & 12 & 0.0029 \\
& NCI109 & 4,127 & 12 & 0.0029 \\
& FRANKENSTEIN & 4,337 & 292 & 0.0673 \\ \midrule
\multirow{2}{*}{OGB Dataset} & ogbg-molhiv & 41,127 & 114 & 0.0027 \\
& ogbg-molpcba & 437,929 & 1,163 & 0.0027 \\ \bottomrule
\end{tabular}
\label{tab:preprocess}

\end{table}


\paragraph{TUDataset}
DD~\cite{dobson2003distinguishing} and PROTEINS~\cite{borgwardt2005protein} are datasets of protein structures labeled according to whether they are enzymes or non-enzymes. 
The main difference between the two datasets is that the nodes are amino acids and protein secondary structures, respectively.
NCI1 and NCI109~\cite{wale2006acyclic} are datasets of chemical compounds labeled according to whether they have activity against non-small cell lung cancer cell line and ovarian cancer cell line, respectively.
The nodes in both datasets are one-hot encodings of atoms.
FRANKENSTEIN~\cite{orsini2015graph} is a chimeric dataset of Mutagenicity and MNIST.
The mutagenic effects of the chemical compounds are measured by Ames test (Salmonella typhimurium reverse mutation assay) and the most frequent atom symbols are remapped as high dimensional MNIST digit vectors.
Compared to NCI1 and NCI109, the unique information of an atom is lost because different images with the same number are used for the same atom (e.g., many different MNIST images ``1'').

\paragraph{OGB Dataset}
The two datasets "ogbg-molhiv" and "ogbg-molpcba" are chemical compound datasets for predicting molecular properties.
They are largest datasets in MoleculeNet~\cite{wu2018moleculenet}.
The performance of "ogbg-molhiv" is measured by ROC-AUC to predict whether a molecule inhibits HIV virus replication.
"ogbg-pcba" is a curated dataset from PubChem BioAssay (PCBA) providing biological activities of small molecules.
The performance of "ogbg-pcba" is measured with average precision because the dataset contains very few positive labels, i.e., extremely biased (positive rate 1.40\%) and various tasks (128 tasks).

\paragraph{Preprocessing Time of Structure Prototypes for SPGP}
The exploration of BCCs and cliques, which are the candidate structure prototypes of SPGP, only needs to be performed once in the preprocessing step, so they do not have a significant impact on model training. 
Each graph for preprocessing takes approximately 0.002 seconds to 0.08 seconds, and detailed information is described in Supplementary Table \hyperref[tab:preprocess]{2}.

\subsection{Brief Explanations of Hierarchical Graph Pooling Methods} \label{appendix:baselines}
Most of hierarchical graph pooling methods are node-selection based graph pooling, which requires an efficient node selection strategy.
To select important nodes, node features and graph structure are needed to compute the node score. From this point of view, SPGP and the existing methods except for DiffPool, which is a node clustering approach rather than node selection, can be summarized as follows and Supplementary Table \hyperref[tab:method-prop]{3}:
\begin{itemize}
    \item TopK~\cite{gao2019graph}: TopK utilizes only node features. Given a node representation matrix $H \in \mathbb{R}^{n \times d}$, TopK computes node scores with a trainable weight matrix $W \in \mathbb{R}^{d \times 1}$ by $\sigma(HW)$, where $\sigma$ is a sigmoid function.
    \item SAGPool~\cite{lee2019self}: Unlike TopK, SAGPool utilizes the amount of information of the node itself and the information of neighboring nodes. To consider node features and graph structure, SAGPool computes node scores by using an additional graph convolution.
    \item GSAPool~\cite{zhang2020structure}: GSAPool separately measures node features and graph structural information. Node importance is calculated by a trainable weight matrix. Structural importance is calculated by an additional graph convolution. Then, node scores are then computed by fusing the two importance via a weighted sum.
    \item ASAP~\cite{ranjan2020asap}: ASAP points out that, since previous methods use information from neighboring nodes when calculating node scores, neighborhoods around nodes with a large amount of information are also redundantly selected. To address this problem, ASAP proposes a new graph convolution called LEConv (Local Extrema Convolution). ASAP also performs structure learning to preserve graph connectivity as node selection can produce isolated nodes
    \item iPool~\cite{gao2021ipool}): Similar to ASAP, iPool also focuses on the relationship between a node and adjacent nodes. iPool computes node scores by L1-norm of the difference between the node representation and the summation of neighborhoods. Therefore, iPool is a parameter-free node selection strategy.
    \item HGP-SL~\cite{zhang2019hierarchical}: HGP-SL is an extension of iPool. The node scores are calculated in the same way as iPool. After that, HGP-SL performs sparse softmax-based structure learning, which requires a lot of computational overhead to deal with isolated nodes.
    \item SPGP (ours): Existing node selection strategies used graph structure information, but are limited to local information. Therefore, we propose, SPGP, a novel node selection method utilizing prior graph structures related to the global structure. In addition, we try to alleviate the isolated node problem without heavy structure learning by enabling information exchange between nodes belonging to the same structure through the structure prototype.
\end{itemize}

\begin{table}[tb]
\caption*{Supplementary Table 3: Properties of node-selection based graph pooling methods.}
\centering
\begin{tabular}{ccccc}
\toprule
\multirow{2}{*}{Method} & \multirow{2}{*}{Node Feature} & \multicolumn{2}{c}{Structure Information} & \multirow{2}{*}{\begin{tabular}[c]{@{}c@{}}Handling\\ Isolated Nodes\end{tabular}} \\ \cmidrule{3-4}
 &  & Local & Global &  \\ \midrule
TopK & \checkmark & & & \\
SAGPool & \checkmark & \checkmark & & \\
GSAPool & \checkmark & \checkmark & & \\
ASAP & \checkmark & \checkmark & & \checkmark \\
iPool & \checkmark & \checkmark & & \\
HGP-SL & \checkmark & \checkmark & & \checkmark \\
SPGP (ours) & \checkmark & \checkmark & \checkmark & \checkmark \\ \bottomrule
\end{tabular}
\label{tab:method-prop}

\end{table}

In summary, the main difference between SPGP and existing methods is that prior graph structures are utilized in the node selection strategy by learnable structure prototypes. 
As shown in Table \ref{tab:performance_comparison} and \ref{tab:ogb} of the main text, SPGP outperforms the other node selection strategies in the same GCN and TopK-based hierarchical pooling framework.

\begin{table}[t]
\caption*{Supplementary Table 4: Hyperparameter tuning Summary for TUDataset.}
\centering
\begin{tabular}{lc}
\toprule
Hyperparameter & Range \\ \midrule
Learning rate & \{0.1, 0.01, 0.005, 0.004, 0.003, 0.002, 0.001\} \\
Batch size & \{64, 128, 256\} \\
Hidden dimension & \{32, 64, 128, 256\} \\
Number of layers & \{2, 3, 4, 5, 6\} \\
Pooling ratio &  \{0.5, 0.8, 0.9, 0.95\} \\
Dropout rate & \{0.0, 0.1, 0.2\} \\
Regularization $\lambda$ & \{0.0, 0.2, 0.4, 0.6, 0.8, 1.0\} \\
Learning rate Decaying Step & \{25, 50\} \\\bottomrule
\end{tabular}
\label{tab:hyperparam-range}

\end{table}

\begin{table*}[t]
\caption*{Supplementary Table 5: Selected hyperparameters for TUDataset.}
\centering
\resizebox{\textwidth}{!}{
\begin{tabular}{lcccccccc}
\toprule
\multicolumn{1}{c}{Dataset} & \begin{tabular}[c]{@{}c@{}}Learning\\ rate\end{tabular} & \begin{tabular}[c]{@{}c@{}}Batch\\ size\end{tabular} & 
\begin{tabular}[c]{@{}c@{}}Hidden\\ dimension\end{tabular} & \begin{tabular}[c]{@{}c@{}}Number of\\ layers\end{tabular} & \begin{tabular}[c]{@{}c@{}}Pooling\\ ratio\end{tabular} & \begin{tabular}[c]{@{}c@{}}Dropout\\ rate\end{tabular} & \begin{tabular}[c]{@{}c@{}}Regularization\\ $\lambda$\end{tabular} & \begin{tabular}[c]{@{}c@{}}Decaying\\ Step\end{tabular} \\ \midrule
DD & 0.002 & 128 & 256 & 6 & 0.95 & 0.1 & 1.0 & 50 \\
PROTEINS & 0.001 & 128 & 64 & 4 & 0.8 & 0.1 & 0.0 & 25 \\
NCI1 & 0.001 & 64 & 256 & 2 & 0.8 & 0.0 & 0.8 & 25\\
NCI109 & 0.001 & 64 & 256 & 2 & 0.8 & 0.0 & 0.2 & 25\\
FRANKENSTEIN & 0.001 & 64 & 64 & 3 & 0.8 & 0.1 & 0.0 & 25 \\ \bottomrule
\end{tabular}%
\label{tab:hyperparam-tu}
}

\end{table*}

\begin{table*}[!tb]
\caption*{Supplementary Table 6: Selected hyperparameters for OGB dataset.}
\centering
\resizebox{\textwidth}{!}{
\begin{tabular}{lcccccccc}
\toprule
\multicolumn{1}{c}{Dataset} & \begin{tabular}[c]{@{}c@{}}Learning\\ rate\end{tabular} & \begin{tabular}[c]{@{}c@{}}Batch\\ size\end{tabular} & 
\begin{tabular}[c]{@{}c@{}}Hidden\\ dimension\end{tabular} & \begin{tabular}[c]{@{}c@{}}Number of\\ layers\end{tabular} & \begin{tabular}[c]{@{}c@{}}Pooling\\ ratio\end{tabular} & \begin{tabular}[c]{@{}c@{}}Dropout rate \\ (Pooling)\end{tabular} &
\begin{tabular}[c]{@{}c@{}}Dropout rate \\ (Conv)\end{tabular} &
\begin{tabular}[c]{@{}c@{}}Regularization\\ $\lambda$\end{tabular} \\ \midrule
ogbg-molhiv & 0.001 & 256 & 256 & 6 & 0.8 & 0.5 & 0.2 & 1.0 \\
ogbg-molpcba & 0.001 & 1,024 & 600 & 3 & 0.8 & 0.1 & 0.3 & 0.0 \\  \bottomrule
\end{tabular}%
\label{tab:hyperparam-ogb}
}

\end{table*}

\subsection{Training \& Hyperparameters}  \label{appendix:training}
\paragraph{TUDataset}
For all our experiments of SPGP on TUDataset, Adam \cite{kingma2014adam} optimizer with $l_2$ regularizer of $5e-4$ is used.
With different 20 seeds, on each seed, 10-fold cross-validation is performed with 80\% for training and 10\% for validation and test each.
Models are trained for 100 epochs with stepLR scheduler (lr decay of 0.1 after every 25 epochs).
The best hyperparameters are selected from a predefined range (Supplementary Table \hyperref[tab:hyperparam-range]{4}) and are provided in Supplementary Table \hyperref[tab:hyperparam-tu]{5}.
The model with the best validation accuracy is selected for testing.
Our models are implemented based on PyTorch Geometric \cite{Fey/Lenssen/2019}.
The experiments are conducted on a Linux server with 126GB memory, 40 CPUs, a single 
NVIDIA GeForce RTX 3090 or RTX 2080 Ti.

\paragraph{OGB Dataset}
For all our experiments of SPGP and baselines, Adam \cite{kingma2014adam} optimizer is used.
With different 20 seeds, we use the scaffold splitting scheme, which is the standard data split scheme provided by OGB dataset.
Models are trained for 100 epochs with early stopping (patience = 30).
The selected hyperparameters are described in Supplementary Table \hyperref[tab:hyperparam-ogb]{6}.
In case of ogbg-molpcba, we use batch size as 1,024 for fast training and utilize BatchNorm instead of LayerNorm.
The evaluator provided by the OGB dataset is used for evaluating model performances.
The experiments are conducted in the same experimental environment as TUDataset.

\subsection{Additional Ablation Study for Structure Prototype} \label{appendix:ablation_proto}
For an additional ablation study, we modify SPGP, which removes the score for the first term, namely the structure prototype, in Equation \ref{eq:prototype_score}, and evaluate it on the TUDataset. 
As shown in Supplementary Table \hyperref[tab:no-proto]{7}, the graph classification performances are significantly degraded without the guidance of node selection through the structure prototype.
Performance degradation is more severe in chemical datasets (NCI1, NCI109, and FRANKENSTEIN) than in biological datasets (DD and PROTEINS). 
This is because the major structure motifs of molecules are often expressed in BCCs or cliques, so the structure prototypes greatly contributes to performance improvement.

\begin{table*}[!tb]
\caption*{Supplementary Table 7: Ablation Study for Effect of the Structure Prototype. Average test accuracy and standard deviation is reported for 20 random seeds.}
\centering
\begin{tabular}{lccc}
\toprule
Dataset & SPGP w/o prototype (A) & SPGP & Diff (A-B)\\ \midrule
DD & $76.69 \pm 0.93$ & $77.76 \pm 0.73$ & $-1.07$\\
PROTEINS & $73.96 \pm 0.61$ & $75.20 \pm 0.74$ & $-1.24$ \\
NCI1 & $74.63 \pm 0.45$ & $78.90 \pm 0.27$ & $-4.27$ \\ 
NCI109 & $73.88 \pm 0.53$ & $77.27 \pm 0.32$ & $-3.39$ \\
FRANKENSTEIN & $61.39 \pm 0.79$ & $68.92 \pm 0.71$ & $-7.53$ \\ \bottomrule
\end{tabular}%
\label{tab:no-proto}
\end{table*}

\subsection{Memory Footprint Test} \label{appendix:mem_test}
We evaluate the GPU memory footprint of SPGP with baselines on the graphs with mean degrees similar to ogbg-molpcba generated by two random graph models: Erdős–Rényi~\cite{gilbert1959random} and Barabási–Albert~\cite{barabasi1999emergence}.
Erdős–Rényi model generates a graph constructed by connected nodes with respect to a probability $p$. 
In this experiment, a graph with $n$ nodes have the probability $p = \delta / n$, which $\delta = 2.16$ is a mean degree per graph in ogbg-molpcba dataset.
Barabási–Albert model generates a scale-free network using a preferential attachment mechanism.
Similar to the experiment of Erdős–Rényi model, newly added nodes preferentially are connected two nodes with high degree nodes.
In both graph models, 1,000 random graphs are generated and the following experimental setup is used for evaluating the memory footprint: batch size = 10, \#  of layers = 3, \# of hidden dimension = 300, pooling ratio = 0.8.
The pooling ratio is chosen as the best performing hyperparameter setting in our experiments on the five small-scale and two large-scale benchmark datasets.
When the pooling ratio is reduced from 0.8 to 0.5, the memory footprint required for each model decreased, but the pattern of increasing memory footprint with the graph size is observed as shown in Figure \ref{fig:mem_test} of main text.
Additionally, decreasing the pooling ratio reduces memory footprint but also results in decreased performance in benchmark dataset experiments.

\begin{figure}[!tb]
\centering
\includegraphics[width=0.6\textwidth]{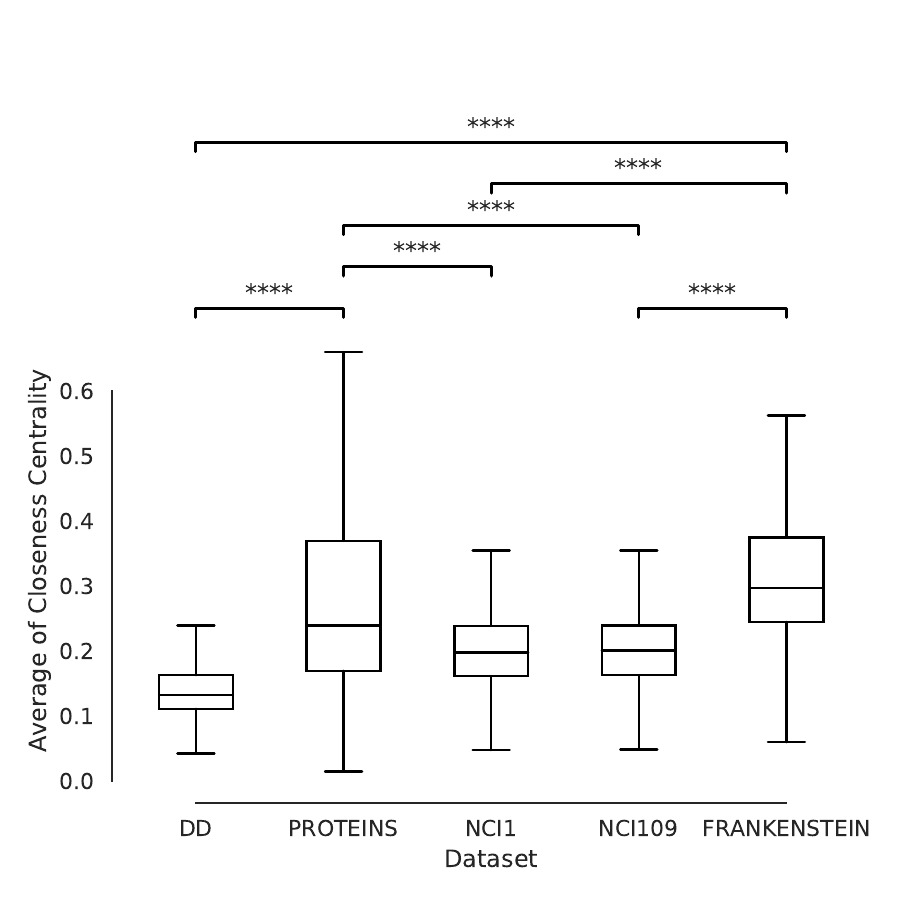} 
\caption*{Supplementary Figure 1: Average closeness centrality of datasets. To compare distribution of two datasets, Mann-Whitney U test two-sided with Bonferroni correction is used. Statistical significance is denoted as symbols:\\$*: \num{0.01}  < p \leq \num{0.05} $, \\  $**: \num{0.001}  < p \leq \num{0.01} $, \\ $***: \num{0.0001}  < p \leq \num{0.001} $, \\ $****: p \leq \num{0.0001} $}
\label{fig:closeness}
\end{figure}

\subsection{Closeness Centrality Analysis of Datasets}  \label{appendix:closness}
Supplementary Figure \hyperref[fig:closeness]{1} shows that the distributions of average of the closeness centrality are different between datasets.
The average of closeness centrality in each graph is measured by averaging the values of closeness centrality of nodes in the graph.
The statistical significance of the distribution difference is calculated by two-sided Mann-Whitney U test with Bonferroni correction.
As shown in Supplementary Figure \hyperref[fig:closeness]{1}, PROTEINS and FRANKENSTEIN datasets have higher closeness centrality values than other datasets.
The high closeness centrality of a node means that the length of the path required to propagate the information of that node to all other nodes is short.
In other words, from a graph point of view, a large average closeness centrality means that the node's information is more globally spread.
For this reason, it may not be necessary to incorporate the auxiliary score in the PROTEINS and FRANKENSTEIN datasets.

\subsection{Additional Evaluation using ogbg-molpcba for GCN/GIN with Virtual Node} \label{appendix:vn_res}
SPGP can be jointly used with advanced graph convolution layers such as GCN/GIN + virtual node. 
We performs SPGP with the official implementation of GCN/GIN + virtual node (VN)~\cite{hu2020ogb} on ogbg-molpbca dataset (450k molecules with 128 classification tasks). 
Supplementary Table \hyperref[tab:vn-res-ogb]{8} shows the average precision of validation and test dataset with two variants of SPGP model. 
First, after the last convolution layer, we add a single SPGP layer (p=0.8). 
Interestingly, we observe that the performance is improved compared to the GCN/GIN + virtual nodes. 
Inspired by this result, we also add one more SPGP layer (p=0.8) after the last second convolution layer, then further performance improvement is also observed in this case.
This is a very interesting result demonstrating the effectiveness of SPGP.

\begin{table*}[!tb]
\caption*{Supplementary Table 8: Performance of SPGP with GCN/GIN+VN.}
\centering
\begin{tabular}{lccc}
\toprule
Model & Validation AP & Test AP \\ \midrule
GCN+VN (leaderboard) & $0.2495 \pm 0.0042$ & $0.2424 \pm 0.0034$ \\
GCN+VN+SPGP (last layer, $p=0.8$) & $0.2556 \pm 0.0021$ & $0.2504 \pm 0.0032$  \\
GCN+VN+SPGP (last 2 layers, $p=0.8$) & $\boldsymbol{0.2579 \pm 0.0044}$ & $\boldsymbol{0.2538 \pm 0.0029}$  \\ \midrule
GIN+VN (leaderboard) & $0.2798 \pm 0.0025$ & $0.2703 \pm 0.0023$  \\
GIN+VN+SPGP (last layer, $p=0.8$) & $0.2818 \pm 0.0032$ & $0.2712 \pm 0.0029$  \\
GIN+VN+SPGP (last 2 layers, $p=0.8$) & $\boldsymbol{0.2848 \pm 0.0018}$ & $\boldsymbol{0.2736 \pm 0.0016}$ \\ \bottomrule
\end{tabular}%
\label{tab:vn-res-ogb}
\end{table*}


\section{Broader Impact} \label{appendix:broader_impact}




Our work introduces a novel graph pooling method that utilizes prior graph structures to learn more accurate graph-level representations.
While most of the existing graph pooling methods focus only on k-hop neighborhood in the graph structure information, our work uses explicit graph structures to provide more structural information to the model.
Our experimental results show that the introduction of various global graph structures can improve graph classification performances.
Thus, our work can provide insight that the use of domain knowledge is important in graph learning, and potentially have positive social impacts in various fields of study related to learning graph representation
For example, biological pathways can be used to obtain patient embeddings from protein-protein interaction networks.
Chemical motifs could play an important role in predicting molecular properties for AI-based drug discovery.

One potential negative societal impact is that the application of our work to specific fields may impose the burden of obtaining well-defined prior graph structures based on domain knowledge.
To mitigate this, we introduce two well-known graph structures in graph theory: biconnected components and cliques.
These two structures are good examples of containing global structural information in general graphs.
Therefore, inspired by our work, it can be a good starting point when conducting research in a specific field.

\end{document}